\def\year{2021}\relax

\documentclass[letterpaper]{article}
\usepackage{lipsum}
\usepackage{xcolor}
\usepackage{soul}
\sethlcolor{black}
\makeatletter
\newif\if@blind
\@blindtrue
\if@blind \sethlcolor{black}\else
   
\fi

\usepackage{aaai21}
\usepackage{times}
\usepackage{helvet}
\usepackage{courier}
\usepackage[hyphens]{url}
\usepackage{graphicx}
\usepackage{graphicx}
\usepackage{natbib}
\usepackage{caption}
\usepackage{multicol}
\usepackage{float}
\usepackage{amsfonts}
\usepackage{amsmath}
\usepackage{amssymb}
\usepackage{amsthm}
\usepackage{mathtools}
\usepackage{algorithm}
\usepackage{algorithmic}
\usepackage{color}

\DeclarePairedDelimiter{\abs}\lvert\rvert
\newcommand\norm[1]{\left\lVert#1\right\rVert}

\DeclareMathOperator*{\argmax}{arg\,max}

\newtheorem{theorem}{Theorem}[]

\newtheorem{lemma}[]{Lemma}

\author{ 
Mahsa Ghasemi, \textsuperscript{\rm 1}\thanks{These authors contributed equally to the manuscript.} Evan Scope Crafts, \textsuperscript{\rm 2 \textasteriskcentered} Bo Zhao, \textsuperscript{\rm 2, 3}  Ufuk Topcu \textsuperscript{\rm 2, 4}\\
}
\affiliations{
\textsuperscript{\rm 1}Department of Electrical and Computer Engineering\\
\textsuperscript{\rm 2}Oden Institute for Computational Engineering and Sciences\\
\textsuperscript{\rm 3}Department of Biomedical Engineering\\
\textsuperscript{\rm 4}Department of Aerospace Engineering and Engineering Mechanics\\
The University of Texas at Austin, Austin, Texas 78712, USA \\
\{mahsa.ghasemi, escopec, bozhao, utopcu\}@utexas.edu
}

\title{Multiple Plans are Better than One: Diverse Stochastic Planning}

\begin{document}

\maketitle

\begin{abstract}
In planning problems, it is often challenging to fully model the desired specifications. In particular, in human-robot interaction, such difficulty may arise due to human's preferences that are either private or complex to model. Consequently, the resulting objective function can only partially capture the specifications and optimizing that may lead to poor performance with respect to the true specifications. Motivated by this challenge, we formulate a problem, called \textit{diverse stochastic planning}, that aims to generate a set of representative --- small and diverse --- behaviors that are near-optimal with respect to the known objective. In particular, the problem aims to compute a set of diverse and near-optimal policies for systems modeled by a Markov decision process. We cast the problem as a constrained nonlinear optimization for which we propose a solution relying on the Frank-Wolfe method. We then prove that the proposed solution converges to a stationary point and demonstrate its efficacy in several planning problems.
\end{abstract}

\section{Introduction}
Solution diversity has value in numerous planning applications, including collaborative systems, reinforcement learning, and preference-based planning. In human groups and, more generally, animal groups, the so-called notion of behavioral diversity leads to the group members' heterogeneous behavior. This heterogeneity ensures that the members learn complementary skills, thus improving the group's overall performance. An agent learning a task in an unknown environment may benefit from inducing diversity in its decisions to explore the environment more efficiently. In planning with unknown preferences, one can use diversity to construct a set of behaviors that are suitable for different preferences. 

Algorithms that use notions of diversity to address one or more of these applications are known as \textit{quality diversity (QD)} algorithms. A key component of QD algorithms is a way to summarize the important properties of different solutions. This description, known as a \textit{behavior characterization}, is used to define diversity-based metrics. Without proper behavior characterization, solutions with trivial differences can have high values of diversity as measured by the resulting metric.

Our work is motivated by planning in settings where, in addition to a known objective, there exist some unknown objectives. The unknown objectives may represent a human user or designer's preference, which is either private or complex to model. In these settings, we propose a QD-based approach to construct a ``representative" --- small and diverse~--- set of near-optimal policies with respect to the known objective and then present that to the human to select from according to their unknown objectives. This approach allows the human to have the ultimate control over the behavior, without requiring prior knowledge of the human's preferences.

Formally, we consider the multi-objective optimization problem of returning a set of feasible policies for an infinite horizon Markov decision process (MDP) that is both near-optimal and diverse. We define the optimality of a set of policies as the sum of each policy's expected average reward in the set. Diversity captures the representativeness of a set of policies. We characterize the behavior of policies using their state-action occupancy measures and quantify diversity by the sum of pairwise divergences between the state-action occupancy measures of the policies in the set.

A key element of our approach is the behavior characterization of policies using their state-action occupancy measures. This characterization is domain-independent and fully encapsulates the dynamics of a given policy. We use this characterization to define the diversity of a set of policies through the pairwise Jensen-Shannon divergences between the occupancy measures. We then define the objective as a linear combination of the sum of the policies' rewards and their diversity. By utilizing the dual of the average cost linear program, we recast our formulation as a constrained optimization problem. We then show that, due to the constraints' linearity, the problem can be solved efficiently using the Frank-Wolfe algorithm. We also prove that the algorithm is guaranteed to converge to a stationary point. Furthermore, in a series of simulations, we evaluate the proposed algorithm's performance and show its efficacy.

\section{Related Work}
Research on the development of QD algorithms has occurred within two different communities. In the field of optimization, perspectives on evolution as a process that finds distinct niches for different species have motivated the use of diversity. Simultaneously, there has been significant interest in the use of diversity to provide high-quality solutions for unknown objectives within the planning community. 

In the optimization community, recent interest in QD algorithms has been driven by the success of the Novelty Search algorithm \cite{lehman2008exploiting}. The original Novelty Search algorithm eschews the use of notions of solution quality entirely; its sole goal is to find a set of solutions that are diverse with respect to some distance measure. Surprisingly, this approach is able to find solutions with better performance on difficult tasks, such as maze navigation, than algorithms relying on an objective function. This result has led to considerable interest in the development of new QD algorithms to address tasks that were previously considered to be too difficult. For a review, see Pugh, Soros, and Stanley (\citeyear{pugh2016quality}). 

The type of behavior characterization used in these works varies and can be domain-dependent. For example, in navigation problems, diversity can be defined using Euclidean distances between points visited. Another approach, used by the popular MAP elites algorithm, is to assume that a domain-dependent behavior characterization is given \cite{mouret2015illuminating}. A promising area of research is the development of new approaches to behavior characterization \cite{gaier2020automating}. 

The success of the Novelty Search and MAP elites algorithms has inspired the use of diversity in reinforcement learning, with the hope that diversity can help avoid poor local minima. Different methods of behavior characterization for policies have been used, including methods based on sequences of actions \cite{jackson2019novelty}, state trajectories \cite{eysenbach2018diversity}, or diversity through determinants of actions in states \cite{parker2020effective}. Similarly to our work, Parker-Holder et al. consider an explicit tradeoff between the quality and diversity of the policies. However, our approach differs in that we leverage knowledge of the system dynamics to characterize policies in a way that includes information about both the states visited and the policy actions, and to develop a solution algorithm with guaranteed convergence to a local minimum.

Behavior characterization has also been a key focus of QD-based work in the planning community. For example, in an approach similar to MAP elites, Myers and Lee (\citeyear{myers1999generating}) and Myers (\citeyear{myers2006metatheoretic}) assume that there is a meta-description of the planning domain. They then define an approach that obtains solutions that are diverse with respect to the meta-description. Another approach to behavior characterization
is through the use of domain landmarks, which are disjunctive sets of propositions that plans must satisfy, such as a set of states that a trajectory must reach before the goal state \cite{hoffmann2001ff}. If the set of landmarks
can be computed, a greedy algorithm can be used to
iteratively select landmarks from the set and find a plan that satisfies the landmark (e.g., reaches a certain state) \cite{bryce2014landmark}. Behavior characterization based on the plan actions, as in the RL community, is also a common technique~\cite{coman2011generating, nguyen2012generating, katz2020reshaping}.

The way behavioral characterization and diversity metrics are incorporated into planning algorithms varies. In some cases, the problem is formulated as maximizing the diversity of the set of solutions \cite{coman2011generating}, or as finding a set of solutions that satisfy a diversity threshold \cite{nguyen2012generating, srivastava2007domain}. In other cases, like our work, there exists both an unknown objective and a known objective, and the problem is formulated in terms of a tradeoff between the diversity of the solution set and the optimality of each of the candidate solutions \cite{coman2011generating, katz2020reshaping, petit2015finding}. Our work is distinct from these approaches because we develop a new method for behavior characterization and consider a stochastic setting modeled as an MDP. In addition,  unlike many QD-based planning algorithms, our approach does not rely on greedy strategies. While greedy algorithms have near-optimality guarantees in some settings, such as when the problem is submodular \cite{bach_2013}, in general no such guarantee exists.

\section{Problem Formulation}
We now overview the required background related to Markov decision processes, occupancy measures, and divergence metrics. Then, we present the main problem as a nonlinear optimization problem over the space of occupancy measures.

\subsection{Preliminaries}

We consider systems whose behavior is modeled by a Markov decision process (MDP). An MDP is a tuple $M = (S, A, P, R)$, where $S$ is a finite set of states, $A$ is a finite set of actions, $P : S \times A \times S \to [0,1]$ is a probabilistic transition function such that for all $s \in S$ and  $a \in A$, $\sum_{s' \in S} P(s'|s,a) = 1$, and $R : S \times A \times S \to \mathbb{R}$ is a reward function. 

A stationary stochastic policy $\pi$ on an MDP is a mapping from the state space to a probability distribution over the actions, formally defined as $\pi : S \times A \to [0,1]$. 
Here we consider only stationary stochastic policies and denote the set of all such policies as $\Pi_{ss}$.

We focus on the class of problems defined over MDPs that aim to maximize the long-run average reward. The long-run average reward of a policy $\pi$ is 
\begin{equation}
    \mathbb{E}_{\tau \sim \pi} \left[\lim_{T \to \infty} \frac{1}{T} \sum_{t=1}^T r(s_t,a_t)\right],
\end{equation}
where the expectation is over all possible trajectory realizations from policy $\pi$, and $s_t$ and $a_t$ are time-indexed states and actions according to a trajectory $\tau$. We assume that the MDP satisfies the weak accessibility (WA) condition.

The occupancy measure of a policy, $\rho^\pi(.,.)$, is defined as the distribution induced by the execution of that policy over the state-action pairs, asymptotically, i.e., 
\begin{equation}
    \rho^\pi(s,a) = \lim_{T\to\infty} \frac{1}{T} \sum_{t=1}^{T} \text{Pr}(s_t=s,a_t=a|\pi).
\end{equation}
The long run behavior of a stationary stochastic policy can be represented using its corresponding occupancy measure. 
An optimal stationary stochastic policy is a policy that maximizes the long-run average reward. It has been shown that under the WA condition, an optimal policy $\pi^*$ can be obtained by solving the
Bellman equation, which can be reformulated as the dual form of a linear program (see Section 4.5 in Volume II of \cite{bertsekas1995dynamic}),
\begin{equation} \label{objfctn}
	\begin{aligned}
		&\max_{\rho} \langle \rho,r \rangle\\
		&\textit{subject to}\\
		&\quad \sum_{a \in A} \rho(s,a) = \sum_{s' \in S} \sum_{a' \in A} P(s|s',a') \rho(s',a')\\
		&\qquad\qquad \textit{for all} \quad s \in S ,\\
		&\quad \sum_{s \in S} \sum_{a \in A} \rho(s,a) = 1,\\
		&\quad \rho(s,a) \ge 0
		\qquad\qquad \textit{for all} \quad s \in S, a \in A
	\end{aligned}
\end{equation}
over occupancy measures $\rho$, where $\langle \rho,r \rangle$ denotes the inner product in the space of $S \times A$, i.e., 
\begin{equation*}
    \langle \rho,r \rangle = \sum_{s \in S} \sum_{a \in A} \rho(s,a) r(s,a).
\end{equation*} 
In particular, an optimal policy $\pi^*$ corresponding to the solution $\rho^*$ from solving the above linear program can be computed by defining
\begin{equation}
	\pi^*(s,a) = \frac{\rho^*(s,a)}{\sum_{a' \in A}\rho^*(s,a')} \;, 
\end{equation}
for all non-transient states. We note that the optimal policy corresponding to $\rho^*$ is not uniquely defined as the choice of action in transient states does not affect the long-run behavior. 

In our problem context, we seek to find a set of policies such that each policy in the set is near-optimal, and the set is representative of the diverse range of near-optimal behaviors. Specifically, we aim to find a small set of policies with cardinality $k \in \mathbb{N}$, and we use $\Pi_k \in \Pi_{ss}^k$ to denote a set of $k$ stationary stochastic policies. 

The state-action occupancy measures provide a natural and domain-independent way to characterize the behavior of policies  and ensure diversity. In particular, by using a pairwise metric of the distance between occupancy measures, we can define a diversity metric for a set of policies. Given that the occupancy measures are probability distributions, a natural choice is the Jensen-Shannon divergence \cite{briet2009properties}. The Jensen–Shannon divergence between two probability distributions $p$ and $q$ is expressed as
\begin{equation}
    \text{JSD}(p \| q) = \frac{1}{2} \text{KL}(p \| m) + \frac{1}{2} \text{KL}(q \| m),
\end{equation}
where $m = 1/2 (p+q)$ is the average distribution, and $\text{KL}(.,.)$ denotes the Kullback–Leibler divergence. Kullback–Leibler divergence for two probability distributions $p$ and $m$, over the same discrete probability space $X$, is defined as
\begin{equation}
    \text{KL}(p \| m) = - \sum_{x \in X} p(x)\log\left(\frac{m(x)}{p(x)}\right).
\end{equation}
We choose the Jensen-Shannon divergence over other probability distribution-based measures because it is symmetric and bounded between zero and one. 

\subsection{Problem Statement}

We aim to design an algorithm that can provide a representative set of polices over an MDP that are near-optimal with respect to a known reward function. In particular, given the stated definitions, the objective is to construct $k$ policies that cumulatively, have high reward and diversity. 
We define the cumulative reward of a set of policies $\Pi_k$ as the sum of their individual accumulated rewards, i.e.,
\begin{equation}
    R(\Pi_k) = \sum_{\pi \in \Pi_k} \mathbb{E}_{\tau \sim \pi} \left[\lim_{T \to \infty} \frac{1}{T} \sum_{t=1}^Tr(s_t,a_t)\right],
\end{equation}
and the their cumulative diversity as the sum of the pairwise Jensen–Shannon divergences between their occupancy measures, i.e.,
\begin{equation}
    D(\Pi_k) = \sum_{\stackrel{\pi_i,\pi_j \in \Pi_k}{i < j}} \text{JSD}(\rho^{\pi_i} \| \rho^{\pi_j}).
\end{equation}
Therefore, given an MDP $M$ and a parameter $k$, the goal is to find a set of policies $\Pi_k \in \Pi_{ss}^k$ with high cumulative reward $R(\Pi_k)$ and high diversity $D(\Pi_k)$.

\section{Proposed Solution}
Our problem statement defines a multi-objective optimization problem that aims to maximize a reward-based objective and a diversity-based objective. A standard method for tackling multi-objective problems is to linearly combine the objectives using judiciously chosen weights. To that end, we first note that the objectives should be independent of the cardinality of the solution set, i.e., the number of policies should not affect the quality of the solution. We address this point by normalizing the reward term by the number of policies, $k$, and the diversity term by the number of unique policy pairs, $k(k-1)/2$. Then, we can define the compound objective function as a linear combination of the normalized reward and diversity. 
The problem of finding $\Pi_k^*$ can thus be cast as finding a solution to the following optimization problem:
\begin{equation}
    \Pi_k^* = \argmax_{\Pi_k \in \Pi_{ss}^k} \frac{1}{k}R(\Pi_k) + \frac{2\lambda}{k(k-1)} D(\Pi_k),
\end{equation}
where $\lambda$ is the tradeoff parameter that controls the relative weightings of the reward and diversity. Using the dual of the linear program for finding an optimal policy, we reformulate the above problem as
\begin{equation}\label{eq:main}
	\begin{aligned}
	    & \max_{\rho_{1:k}} f(\rho_{1:k}) \\
		&\textit{subject to}\\
		&\quad \sum_{a \in A} \rho_i(s,a) = \sum_{s' \in S} \sum_{a' \in A} P(s|s',a') \rho_i(s',a')\\
		&\qquad\qquad \textit{for all} \quad i \in [k], s \in S ,\\
		&\quad \sum_{s \in S} \sum_{a \in A} \rho_i(s,a) = 1 \qquad\qquad \textit{for all} \quad i \in [k],\\
		&\quad \rho_i(s,a) \ge 0 \qquad\qquad \textit{for all} \quad i \in [k], s \in S, a \in A,
	\end{aligned}
\end{equation}
where 
\begin{equation*}
    f(\rho_{1:k}) = \frac{1}{k}\sum_{i \in [k]} \langle \rho_i,r \rangle + 
		\frac{2\lambda}{k(k-1)}\sum_{\stackrel{i,j \in [k]}{i < j}} \text{JSD}(\rho_i \| \rho_j),
\end{equation*}
$\rho_{1:k} = \{\rho_1,\rho_2,\ldots,\rho_k\}$ denotes the $k$ occupancy measures corresponding to the $k$ policies, and $[k] = \{1,2,\ldots,k\}$.

The reformulated version is a constrained optimization problem with linear constraints and a nonlinear (and nonconcave) objective function. In general, this problem does not have a unique global solution. For instance, any permutation of the $k$ policies in an optimal solution will result in another optimal solution. Nonetheless, one can seek solutions that can at least converge to local stationary points. 

\subsection{Projected Gradient Ascent}

The first optimization method that we consider is projected gradient ascent (PGA)~\cite{boyd2004convex}. PGA iteratively applies a gradient update followed by a projection step. Let $\mathcal{P}(\tilde{\rho},M)$ denote the projection operator projecting $\tilde{\rho}$ onto the space of feasible occupancy measures for MDP $M$, i.e., it returns the solution to the optimization problem
\begin{equation}\label{eq:proj}
	\begin{aligned}
		& \min_\rho \mathcal{D}(\rho,\tilde{\rho})\\
		&\textit{subject to}\\
		&\quad \sum_{a \in A} \rho(s,a) = \sum_{s' \in S} \sum_{a' \in A} P(s|s',a') \rho(s',a')\\
		&\qquad\qquad \textit{for all} \quad s \in S ,\\
		&\quad \sum_{s \in S} \sum_{a \in A} \rho(s,a) = 1, \\
		&\quad \rho(s,a) \ge 0 \qquad\qquad \textit{for all} \quad s \in S, a \in A.
	\end{aligned}
\end{equation}

We choose the projection metric to be $\ell_2$-norm, i.e., $\mathcal{D}(\rho,\tilde{\rho}) \coloneqq \norm{\rho - \tilde{\rho}}_2^2$.
The details of the PGA algorithm are outlined in Algorithm~\ref{alg:pga}. Let $\Delta_M$ represent the space of feasible occupancy measures defined by the constraints in~\eqref{eq:main}. We initialize the occupancy measures by first defining random policies for the given MDP, i.e., policies with random probability distributions over actions in each state. The algorithm terminates once the convergence criteria are met, e.g., the gradient mapping \cite{nesterov2013introductory} defined as 
\[h^t:= \frac{1}{\eta^t} (\rho_{1:k}^{t+1}-\rho_{1:k}^{t+1/2})\] 
hits a target threshold or the number of iterations exceeds a prespecified number.

\begin{algorithm}[t]
	\caption{Diverse Stochastic Planning with Projected Gradient Ascent}
	\label{alg:pga}
	\begin{algorithmic}[1]
		\STATE \textbf{Input:} An MDP $M = (S,A,P,R)$, number of policies $k$, tradeoff parameter $\lambda$, step size $\eta^t$, maximum number of iterations $T$, convergence tolerance $\epsilon$\vspace{1mm}
		\STATE \textbf{Output:} Occupancy measures $\rho^T_1, \rho^T_2, \ldots, \rho^T_k$\vspace{1mm}
		\STATE Initialize $\rho^0_1, \dots, \rho^0_k$ by randomly sampling from $\Delta_M$\vspace{1mm}
		\FOR {$t = 0,\dots, T$}\vspace{1mm}
		\FOR {$i = 1,\dots, k$}\vspace{1mm}
		\STATE Find $\rho_i^{t + 1/2} \leftarrow \rho_i^t + \eta^t \nabla_{\rho_i} f(\rho_{i:k}^{t})$\vspace{1mm}
		\STATE Compute $\rho_i^{t+1} \leftarrow \mathcal{P}\left({\rho_i^{j + 1/2}, M}\right)$\vspace{1mm}
		\ENDFOR\vspace{1mm}
		\STATE Find the gradient mapping $h^t:= \frac{1}{\eta^t} (\rho_{1:k}^{t+1}-\rho_{1:k}^{t+1/2})$\vspace{1mm}
		\IF {$h^t \leq \epsilon$}\vspace{1mm}
		\STATE Return $\rho_1^{t+1} , \rho_2^{t+1} , \ldots, \rho_k^{t+1} $\vspace{1mm}
		\ENDIF\vspace{1mm}
		\ENDFOR\vspace{1mm}
        \STATE Return $\rho^T_1, \rho^T_2, \ldots, \rho^T_k$\vspace{1mm}
	\end{algorithmic}
\end{algorithm}

It is worth noting that the projection step~\eqref{eq:proj} is much simpler than the original constrained nonlinear optimization in~\eqref{eq:main} because it is a convex optimization problem, and hence, amenable to efficient solutions.

\subsection{Frank-Wolfe Algorithm}

Even though PGA can decouple the projection step for each policy, it still has to solve a convex optimization problem for each policy at each iteration. To avoid this complexity, we propose the use of the Frank-Wolfe (FW) algorithm~\cite{frank1956algorithm}. Every iteration of FW aims to move toward a minimizer of the linear approximation of the original objective function at the current point. Due to this fact, it has gained more popularity for optimization problems with structured constraint set. In particular, the linearity of the constraints in~\eqref{eq:main} turns every iteration of FW into a linear optimization problem. We implement FW with adaptive step sizes~\cite{lacoste2016convergence} and backtracking line search, as presented in Algorithm~\ref{alg:fw}. 
At iteration $t$, the algorithm finds a feasible point $s^t$ within the set of feasible policies, $\Delta_M$, that minimizes the linear approximation of $f(\rho_{1:k})$ at the current point. Then, it moves in the direction of $s^t$ by a step size $\gamma^t$ that is computed using a line search. We efficiently implement the line search using backtracking. The algorithm terminates once the FW gap defined as
\[g^t \coloneqq \langle d^t,\nabla_{\rho_{1:k}} f(\rho_{1:k}^{t}) \rangle\]
falls below a given tolerance or the number of iterations reaches a prespecified number $T$.

\begin{algorithm}[t]
	\caption{Diverse Stochastic Planning with Frank-Wolfe Algorithm}
	\label{alg:fw}
	\begin{algorithmic}[1]
		\STATE \textbf{Input:} An MDP $M = (S,A,P,R)$, number of policies $k$, tradeoff parameter $\lambda$, maximum number of iterations $T$, convergence tolerance $\epsilon$\vspace{1mm}
		\STATE \textbf{Output:} Occupancy measures $\rho^T_1, \rho^T_2, \ldots, \rho^T_k$\vspace{1mm}
		\STATE Initialize $\rho^0_1, \dots, \rho^0_k$ by randomly sampling from $\Delta_M$\vspace{1mm}
		\FOR {$t = 0,\dots, T$}\vspace{1mm}
		\STATE Compute $s^t \coloneqq \argmax_{s \in \Delta} \langle s,\nabla_{\rho_{1:k}} f(\rho_{1:k}^{t}) \rangle$\vspace{1mm}
		\STATE Find FW update direction $d^t \coloneqq s^t - \rho_{1:k}^t$\vspace{1mm}
		\STATE Find FW gap $g^t \coloneqq \langle d^t,\nabla_{\rho_{1:k}} f(\rho_{1:k}^{t}) \rangle$\vspace{1mm}
		\IF {$g^t \leq \epsilon$}\vspace{1mm}
		\STATE Return $\rho_1^t, \rho_2^t, \ldots, \rho_k^t$\vspace{1mm}
		\ELSE\vspace{1mm}
		\STATE Compute $\gamma^t = \argmax_{\gamma \in [0,1]} f(\rho_{1:k}^t + \gamma d^t)$\vspace{1mm}
		\ENDIF\vspace{1mm}
		\STATE $\rho_{1:k}^{t+1} \leftarrow \rho_{1:k}^t + \gamma^t d^t$\vspace{1mm}
		\ENDFOR\vspace{1mm}
        \STATE Return $\rho_1^T, \rho_2^T, \ldots, \rho_k^T$\vspace{1mm}
	\end{algorithmic}
\end{algorithm}

\section{Theoretical Guarantees}
Next, we prove that by applying PGA and FW on a slightly revised problem one can establish non-asymptotic convergence rates to a stationary point. 

Let $\Delta_{M,\delta} = \Delta_M \cap \{\rho_{1:k} | \rho_i \geq \delta, \forall i \in [k]\}$ for some $\delta > 0$ represent a restricted space for occupancy measures. In the next lemma, we prove that the gradient of the objective function $f(\rho_{1:k})$ is Lipschitz continuous over the restricted space $\Delta_{M,\delta}$. 
\begin{lemma}\label{lem:lip}
Let $\delta \in (0,1)$. The gradient of the objective function $f(\rho_{1:k})$ defined in~\eqref{eq:main} is $L$-Lipschitz over $\Delta_{M,\delta}$. That is,  
\begin{equation*}
\begin{aligned}
        \|\nabla f(\rho_{1:k})- \nabla f(\rho_{1:k}')\|_2 &\leq  L\|\rho_{1:k}- \rho_{1:k}'\|_2,\\ \forall \text{ } \rho_{1:k}, \rho_{1:k}' \in \Delta_{M,\delta}, \quad & L := \lambda \frac{1+\delta}{4\delta^2}.
\end{aligned}
\end{equation*}
\end{lemma}

\begin{proof}
    First, we note that the linear term of $f(\rho_{1:k})$ does not contribute to the Lipschitzness. Moreover, the diversity term has been normalized. Therefore, if we can show that $\text{JSD}(\rho_i \| \rho_j)$ has a $L'$-Lipschitz gradient, then we can conclude that $f(\rho_{1:k})$ has a $\lambda L'$-Lipschitz gradient. To show that $\text{JSD}(\rho_i \| \rho_j)$ has Lipschitz gradient, we start by computing an entry of its Hessian $\nabla^2 \text{JSD}(\rho_i \| \rho_j)$. Let $x = (s,a) \in X$ be an arbitrary state-action pair where $X= S \times A$. Then, we have 
    \begin{equation*}
    \begin{aligned}
        \text{JSD}(\rho_i \| \rho_j) &= \frac{1}{2}\sum_{x \in X} \rho_i(x) \log \frac{2\rho_i(x)}{\rho_i(x)+\rho_j(x)}\\
        &+ \frac{1}{2}\sum_{x \in X} \rho_j(x) \log \frac{2\rho_j(x)}{\rho_i(x)+\rho_j(x)}.
    \end{aligned}
    \end{equation*}
    Taking the derivative with respect to an arbitrary $x$, we obtain
    \begin{equation*}
        \frac{\partial \text{JSD}(\rho_i \| \rho_j)}{\partial \rho_i(x)} = 
        \frac{1}{2} \log \frac{2\rho_i(x)}{\rho_i(x) + \rho_j(x)}.
    \end{equation*}
    By some straightforward calculation, one can see that the Hessian is sparse. More specifically, it holds that
    \begin{align*}
        \frac{\partial^2 \text{JSD}(\rho_i \| \rho_j)}{\partial \rho_i(x)\partial \rho_i(x)} &= \frac{\rho_j(x)}{2\rho_i(x)(\rho_i(x)+\rho_j(x))},\\
        \frac{\partial^2 \text{JSD}(\rho_i \| \rho_j)}{\partial \rho_i(x')\partial \rho_i(x)} &= 0,\\
        \frac{\partial^2 \text{JSD}(\rho_i \| \rho_j)}{\partial \rho_j(x)\partial \rho_i(x)} &= -\frac{1}{2(\rho_i(x)+\rho_j(x))},\\
        \frac{\partial^2 \text{JSD}(\rho_i \| \rho_j)}{\partial \rho_j(x')\partial \rho_i(x)} &= 0,\\
    \end{align*}
    where $x \neq x'$.
    Therefore, given that only two entries of each column of the Hessian are nonzero, using the Gershgorin circle theorem \cite{horn2012matrix}, we can show that
    \begin{equation*}
    \begin{aligned}
        \|\nabla^2 \text{JSD}(\rho_i \| \rho_j)\|_2 
        &\leq \abs[\Big]{\frac{\rho_j(x)}{2\rho_i(x)(\rho_i(x)+\rho_j(x))}}\\
        &+ \abs[\Big]{-\frac{1}{2(\rho_i(x)+\rho_j(x))}}\\
        &\leq \frac{1}{4\delta^2} + \frac{1}{4\delta} = \frac{1+\delta}{4\delta^2}.
    \end{aligned}
    \end{equation*}
    Hence, $\nabla \text{JSD}(\rho_i \| \rho_j)$ is $(1+\delta)/(4\delta^2)$-Lipschitz and consequently, $\nabla f(\rho_{1:k})$ is $\lambda(1+\delta)/(4\delta^2)$-Lipschitz.
\end{proof}

The following two theorems establish that since the gradient is Lipschitz, both PGA and FW are guaranteed to converge to a stationary point.

\begin{theorem}[Theorem 6.5 \cite{lan2020first}]
Define the minimal gradient mapping of the PGA algorithm as $\|\tilde{h}\|_2 := \min_{0\leq t\leq T} \|h^t\|_2$ encountered by the iterates during the algorithm until the $T\textsuperscript{th}$ iteration. Suppose that the stepsizes $\{\eta^t\}$ in the PGA scheme are
chosen such that $0<\eta^t<2/L$, where $L$ is the Lipschitz constant of the gradient of $f(\rho_{1:k})$ on $\Delta_{M,\delta}$. Then it holds that
\begin{equation}
    \|\tilde{h}\|_2 \leq \sqrt{\frac{f_\delta^\ast-f(\rho_{1:k}^{0})}{\sum_{t=0}^T\eta^t(1-L\eta^t/2)}},
\end{equation}
where $f_\delta^\ast$ denotes the optimal solution of \eqref{eq:main} over the restricted domain.
In particular, If $\eta^t = 1/L$, then
\begin{equation}
    \|\tilde{h}\|_2 \leq \sqrt{\frac{2L(f_\delta^\ast-f(\rho_{1:k}^{0}))}{T+1}}.
\end{equation}
\end{theorem}

\begin{theorem}[\cite{lacoste2016convergence}]
Define the minimal FW gap as $\|\tilde{g}\|_2 := \min_{0\leq t\leq T} \|g^t\|_2$ encountered by the iterates during the algorithm until the $T\textsuperscript{th}$ iteration. Consider running the FW algorithm with the adaptive stepsize strategy specified in Line 11 of Algorithm~\ref{alg:fw}. Then, it holds that
\begin{equation}
     \|\tilde{g}\|_2 \leq \frac{\max\{2(f_\delta^\ast-f(\rho_{1:k}^{0})),\mathrm{Dim}(\Delta_{M,\delta})^2L\}}{\sqrt{T+1}},
\end{equation}
where $L$ is the Lipschitz constant of the gradient of $f(\rho_{1:k})$, and $f_\delta^\ast$ denotes the optimal solution of \eqref{eq:main} over the restricted domain.
\end{theorem}

\section{Experiments and Results}
\begin{figure}
    \centering
    \includegraphics[width = .23\textwidth]{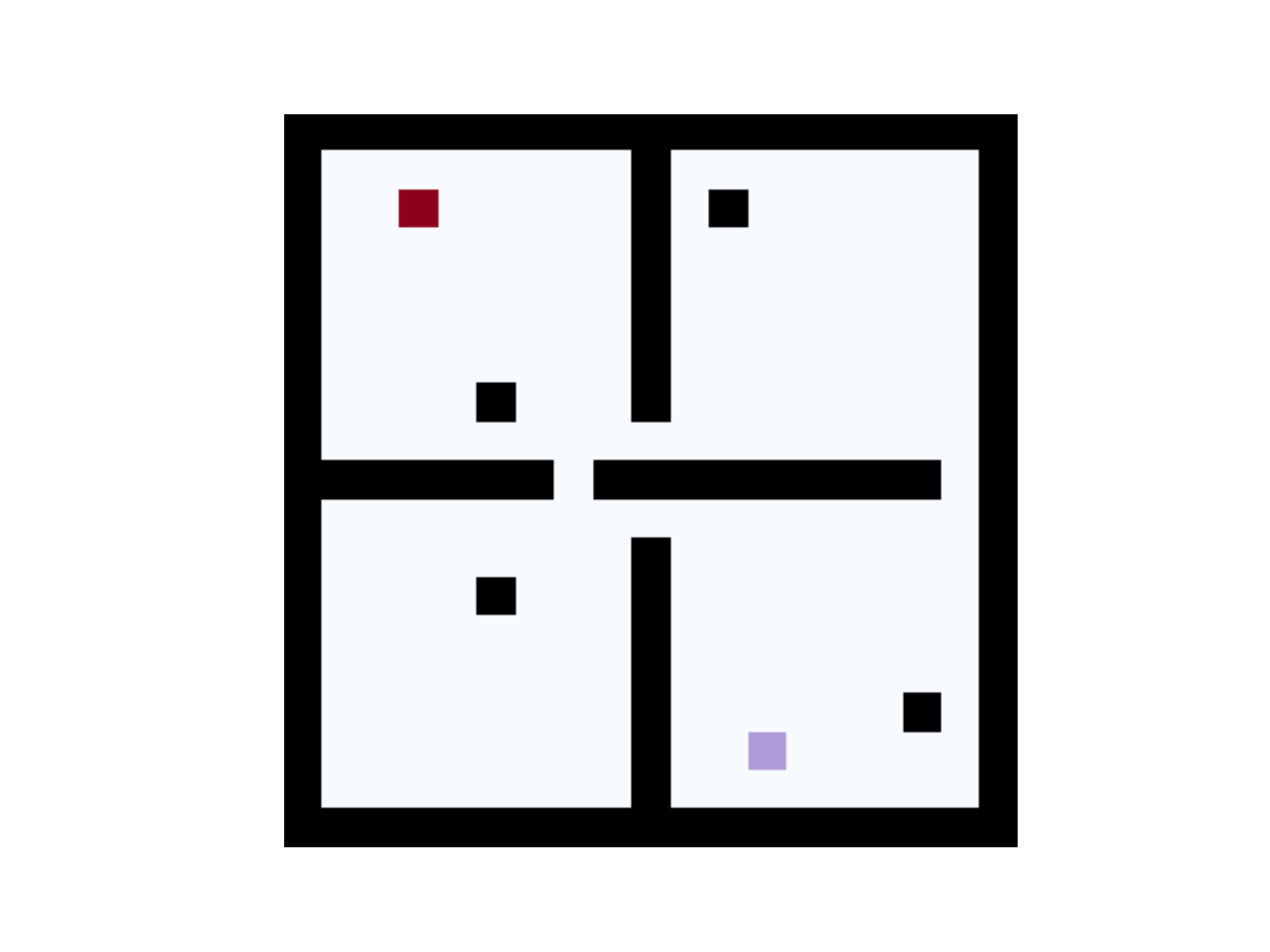} 
    \includegraphics[width = .23\textwidth]{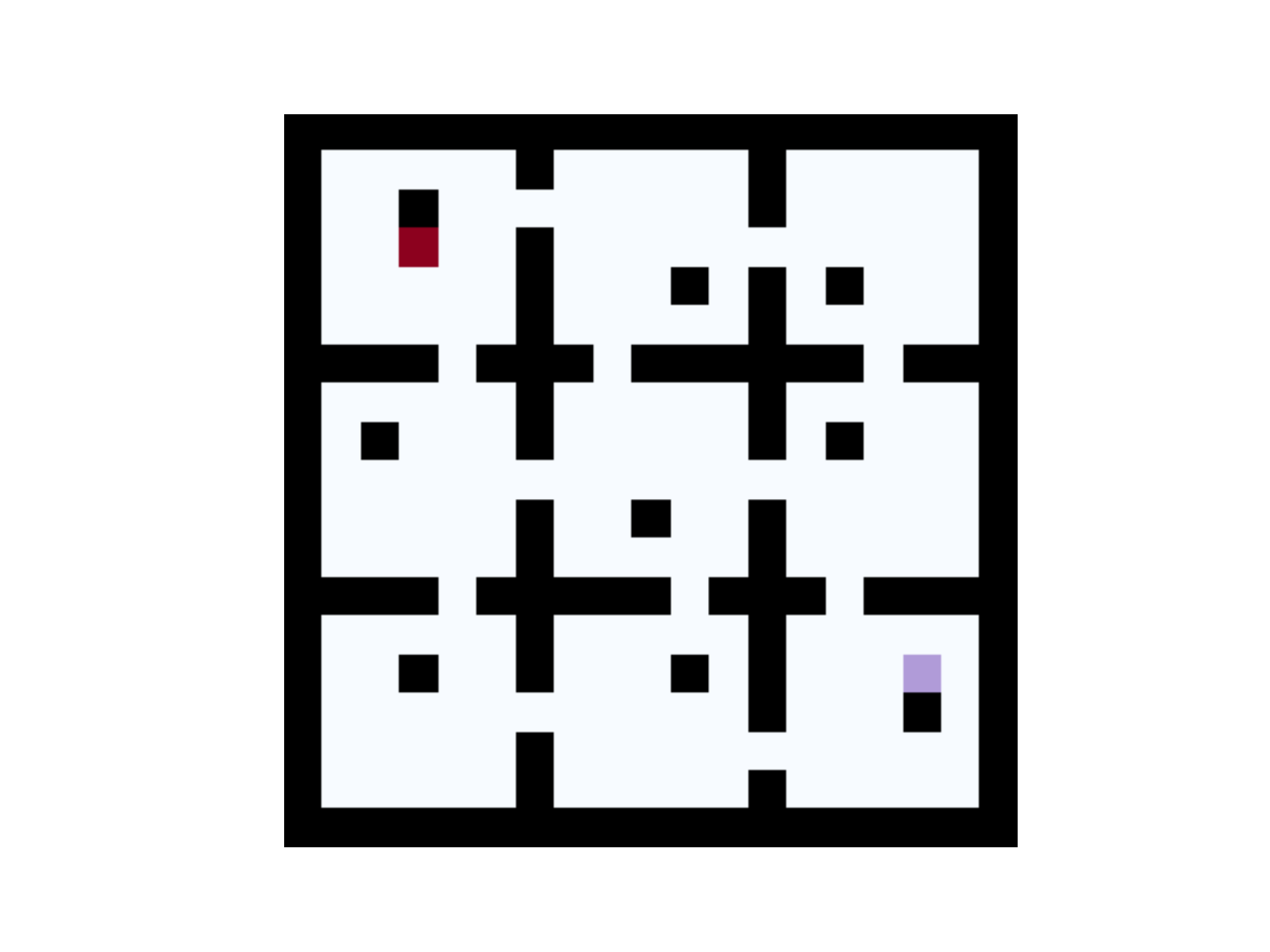} \\
    \hspace{1.8cm} (a) \hfill (b) \hspace{1.8cm}
    \caption{Example realizations of the (a) four-room and (b) nine-room grid worlds. Obstacles and walls are black, the initial state is red, and the goal state is purple.}
    \label{WorldDesign}
\end{figure}

We evaluate the performance of our proposed approach using grid worlds. 

\subsection{Grid World Design}

The grid worlds are two-dimensional nineteen-by-nineteen rectangular grids. The state of the agent is its current position. The agent receives a large reward for reaching a defined goal state, after which it transitions back to an initial state. In all states other than the goal, the agent has five choices of actions that correspond to moving down, left, up, right, or stopping and staying in the same place. The `stop' action is deterministic; for all other actions, the probability of transitioning to the desired successor state is given by a `correct transition' hyper-parameter, denoted $\alpha$. If the transition is not successful, the agent will transition to another neighboring state randomly. The environment is filled with obstacles and reaching an obstacle state results in a large penalty. We note that this design could represent, for example, the environment of a robot vacuum or guard robot. 

We define two different grid world types. In the first, which we refer to as the four-room grid world, the environment consists of four eight-by-eight rooms arranged in a two-by-two grid. There is a -200 reward for reaching states with walls and obstacles, a reward of 400 for reaching the goal, and rewards of -4 in all other states to shape the agent behavior. In the second grid world, there are nine total five-by-five rooms arranged in a three-by-three grid. We refer to this setup as the nine-room grid world. In this setup, there is a -40 reward for reaching walls and obstacles, a reward of 200 for reaching the goal, and a reward of -1.2 in all other states. Both environments have walls separating adjacent rooms with a single door, represented by a hole in the wall, linking the rooms. There is a single additional obstacle placed within each room. In both worlds, the initial state is located in the top left room, and the goal state is located in the bottom right room.

To evaluate the robustness of our proposed approach, we test the performance over a range of trials. In each trial, the location of the agent and the goal state within their respective rooms, the locations of the doors in each wall, and the locations of the obstacles within the rooms are randomized. Figure \ref{WorldDesign} shows a single trial grid world for both the four- and nine-room setups. 

\subsection{Frank-Wolfe and Projected Gradient Ascent}

\begin{figure}
    \centering
    \begin{tabular}{c| c  c c}
    Opt. Alg. & Reward/policy & Diversity & Runtime (s) \\
    \hline
    PGA &  -39.02 & 0.34 & 1488.76 \\
    FW & 13.24 & 0.50 & 26.90 \\
    \end{tabular}
    \caption{Performance of PGA and FW on the four-room grid world with $\lambda = 8$, $k = 2$, and $\alpha = .95$. Results are averaged over ten trials.}
    \label{FrankWolfevsPGA}
\end{figure}

Here we compare the performance of the FW and PGA optimization algorithms. We terminate PGA when the maximum iteration number is reached or when the difference in norms between consecutive solutions falls below a tolerance threshold of .01. We use Sequential Least Squares Programming \cite{nocedal2006numerical} to solve the projection step for each iteration. The Sequential Least Squares Programming algorithm terminates after ten iterations or when a stationarity condition is met. We implement the FW algorithm with a shrinkage factor of $\gamma = .5$ for the backtracking line search. The FW algorithm terminates when the Frank-Wolfe gap falls below a tolerance of .001 or when the maximum iteration number is reached. We set the maximum iteration number for both approaches as $T = 30$. 

We evaluate performance using the four-room grid-world with the correct transition parameter $\alpha = .95$, the number of policies in the return set $k = 2$, and the tradeoff parameter $\lambda = 8$. Figure \ref{FrankWolfevsPGA} shows the average performance over ten trials. FW is clearly superior to PGA in both performance and computational efficiency. This is because PGA involves solving a constrained least-squares optimization problem for each policy at each iteration to project the policies back onto the feasible space. Even small errors in the projection can considerably deteriorate the near-optimality and diversity of the policies. In contrast, FW only requires a linear program to be solved at each iteration. The solution to the linear program lies in the feasible space by construction and thus there are no issues with stability. We use FW as the optimization algorithm in the subsequent experiments.

\subsection{Role of the Tradeoff Parameter}

\begin{figure*}
\centering
    \includegraphics[width = .45\textwidth]{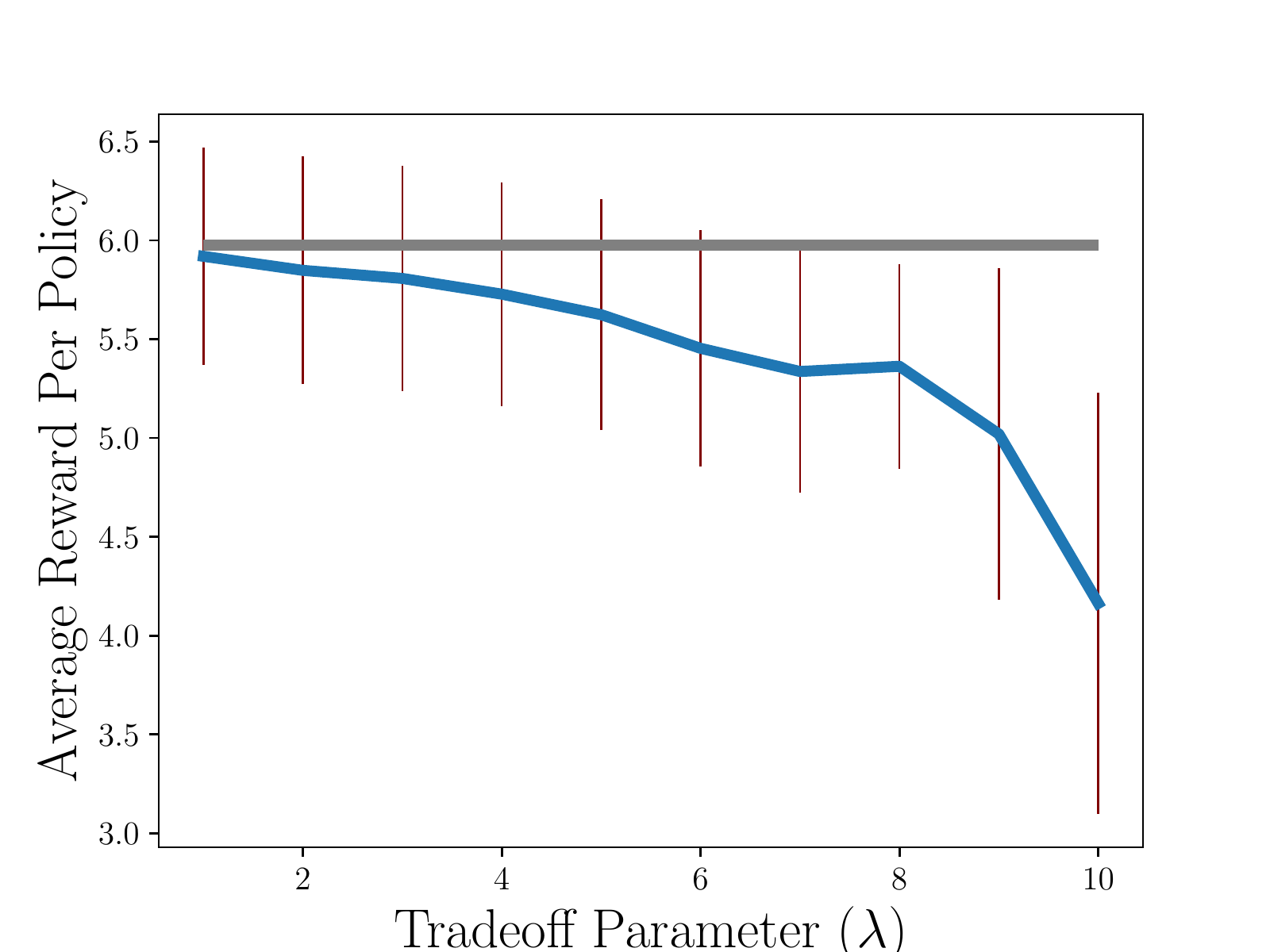}
    \includegraphics[width = .45\textwidth]{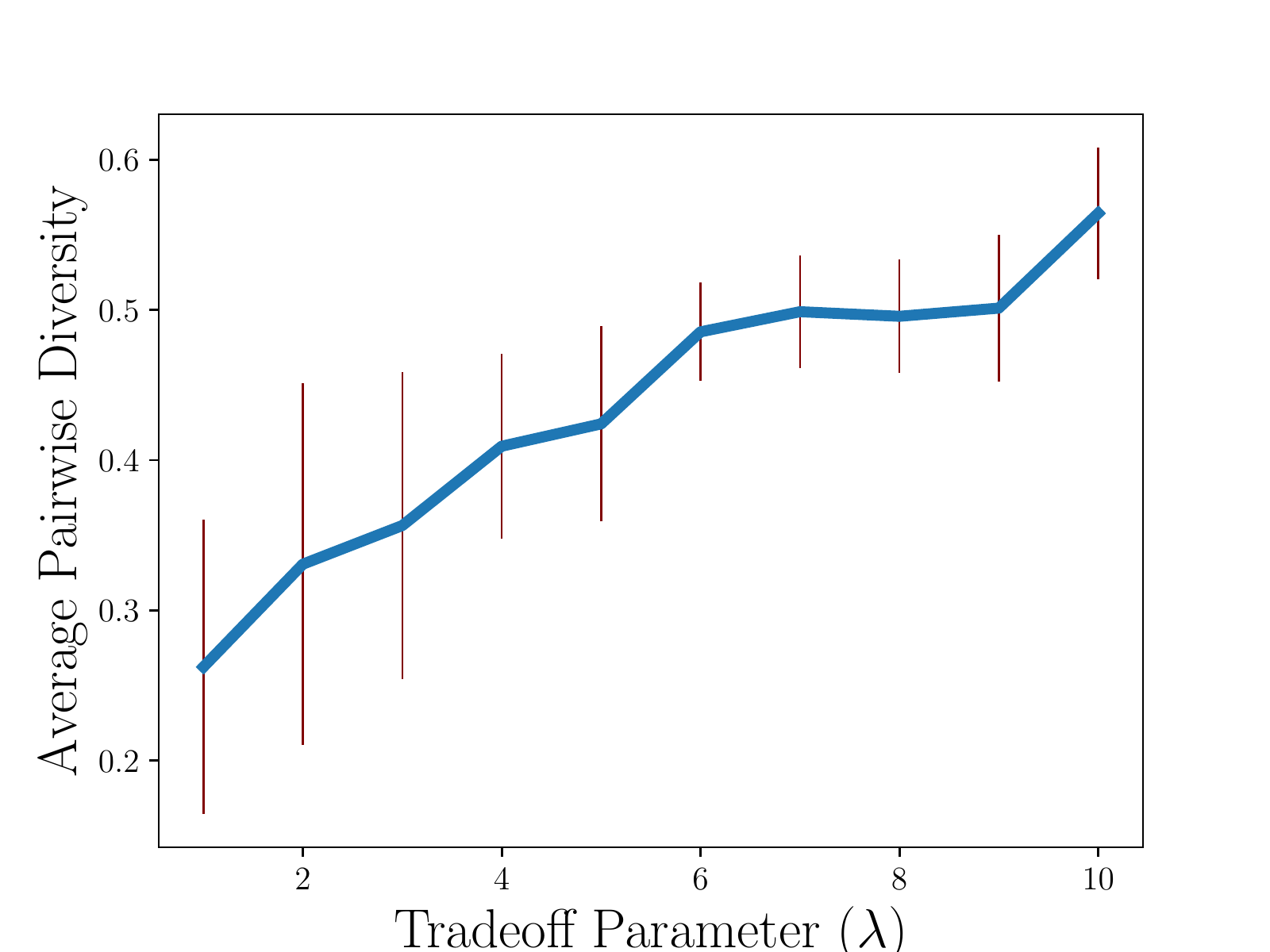} \\
    (a) \hspace{7.5cm} (b) \\
    \caption{The reward and diversity of the policies found as a function of the tradeoff parameter in the nine-room grid-world. Mean and standard deviation displayed are computed for ten trials. 
    (a) Average reward per policy. Reward for optimal policy is displayed in grey.
    (b) Average pairwise diversity.
    }
    \label{TradeoffPlots}
\end{figure*}

\begin{figure*}[!htb]
\centering
    \includegraphics[trim={2cm 5cm 1.5cm 5cm},clip, width =
    .95\textwidth]{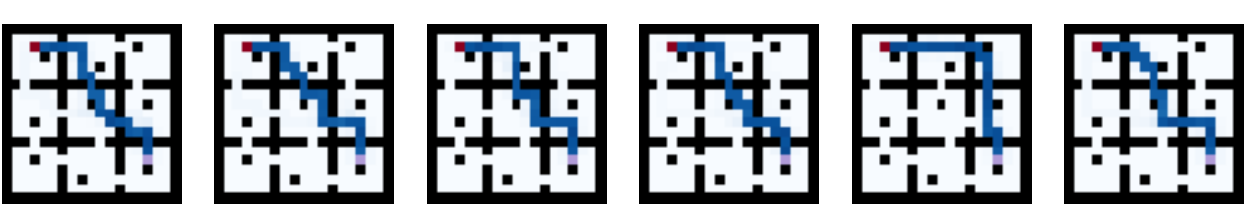}  \\
    (a) \\
    \includegraphics[trim={2cm 5cm 1.5cm 5cm},clip, width = .95\textwidth]{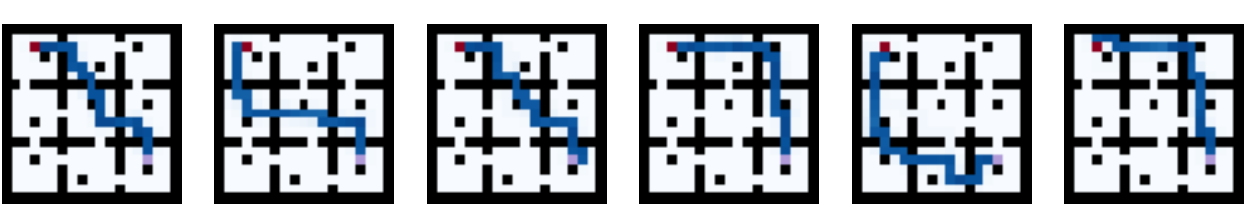}   \\
    (b)
    \caption{State occupancy maps from the policies obtained using the proposed approach for a single nine-room trial and two different tradeoff parameter values. Obstacles and walls are black, the initial state is red, the goal state is purple, and the state occupancy probability is blue. (a) $\lambda = 2$. (b) $\lambda = 8$. }
        \label{TradeoffOcuppancy}
\end{figure*}

The tradeoff parameter plays a crucial role in ensuring a proper balance between the near-optimality of the candidate solutions and the diversity of the set of solutions. Testing the performance as a function of the tradeoff parameter provides important insights into the performance and properties of our proposed approach. Here we evaluate the performance for a range of tradeoff parameters using the nine-room grid world. We fix the correct transition parameter $\alpha = .95$, and set the number of policies in the return set as $k = 6$. This $k$ value is the number of unique door combinations the agent can take to reach the goal without cycling or other undesirable behavior. 

Figure \ref{TradeoffPlots} shows the average reward per policy and the average pairwise diversity over ten trials. As expected, the pairwise diversity shows a marked increase as a function of the tradeoff parameter. The average reward decreases slightly as the tradeoff parameter increases until it begins to fall sharply around $\lambda = 8$. This is the point where it becomes optimal in some trials to find policies that do not reach the goal but have maximal diversity. Up to this point, our approach is still able to find increasingly diverse near-optimal solutions. 

This behavior can also be observed in Figure \ref{TradeoffOcuppancy}, which shows sample state occupancy maps for a small range of lambda values and a single trial. The state occupancy measure is the long-run expected probability of being in a given state $s$, i.e., $\rho(s) = \sum_{a \in A} \rho(s, a)$. With $\lambda = 2$, our approach finds several policies with nearly identical behavior. With $\lambda = 8$, the algorithm finds policies that utilize increasingly diverse strategies to reach the goal and traverse through many of the doors and rooms in the grid world. Note, however, that even with $\lambda = 8$, policies one and three and policies four and six have relatively similar behavior as they utilize the same door combinations to reach the goal. This can be explained by the fact that our approach finds only a local minimum in the loss landscape, and by the fact that even with high values of $\lambda$ the configuration of the doors and obstacles can limit the number of meaningfully distinct near-optimal policies.

\begin{figure*}[!htb]
    \centering
    \begin{multicols}{2}
    \includegraphics[trim = {0, 0, 0, 1cm}, clip, width = .45\textwidth]{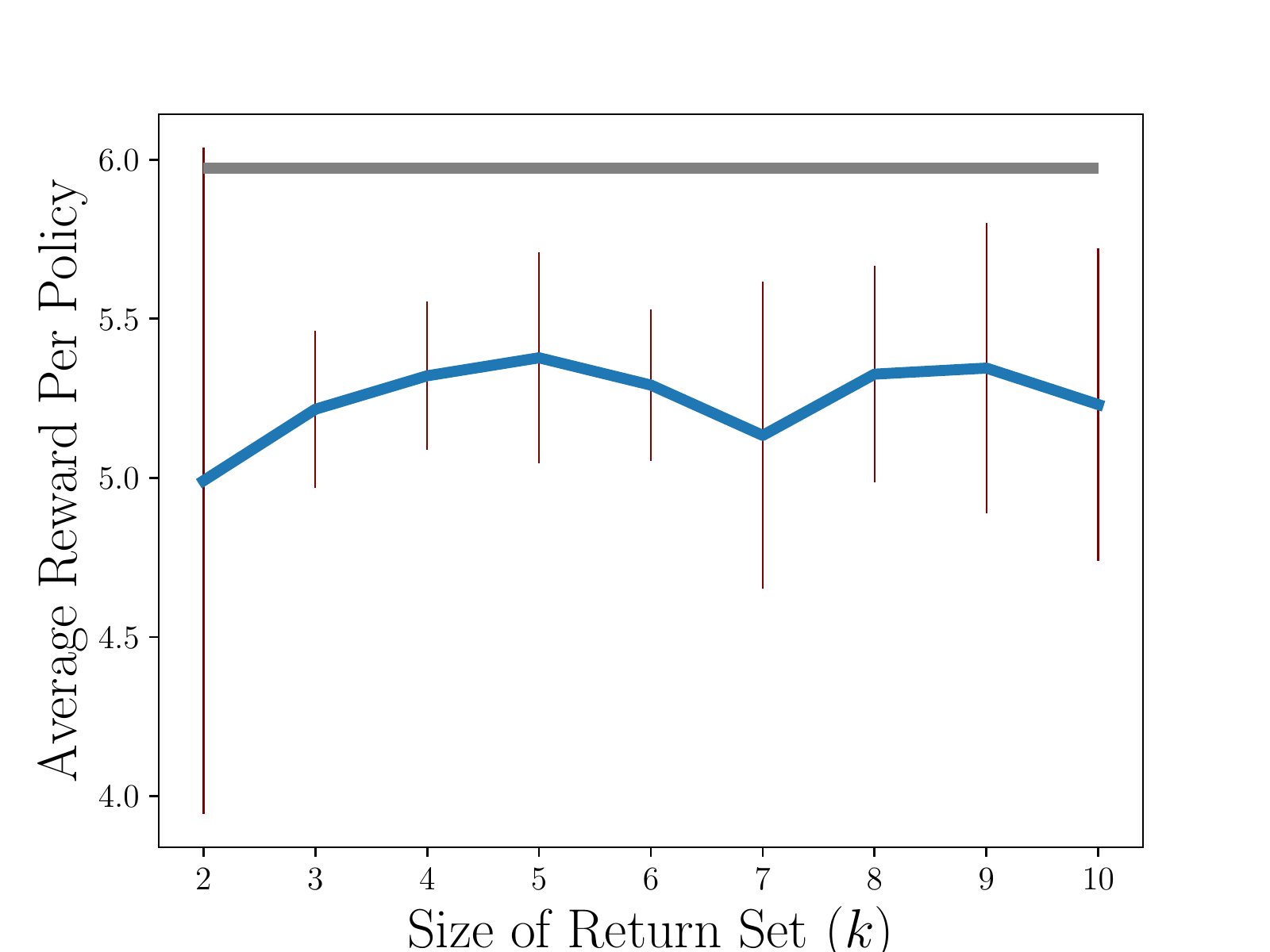} \\
        (a) \\
    \includegraphics[trim = {0, 0, 0, 1cm}, clip, width = .45\textwidth]{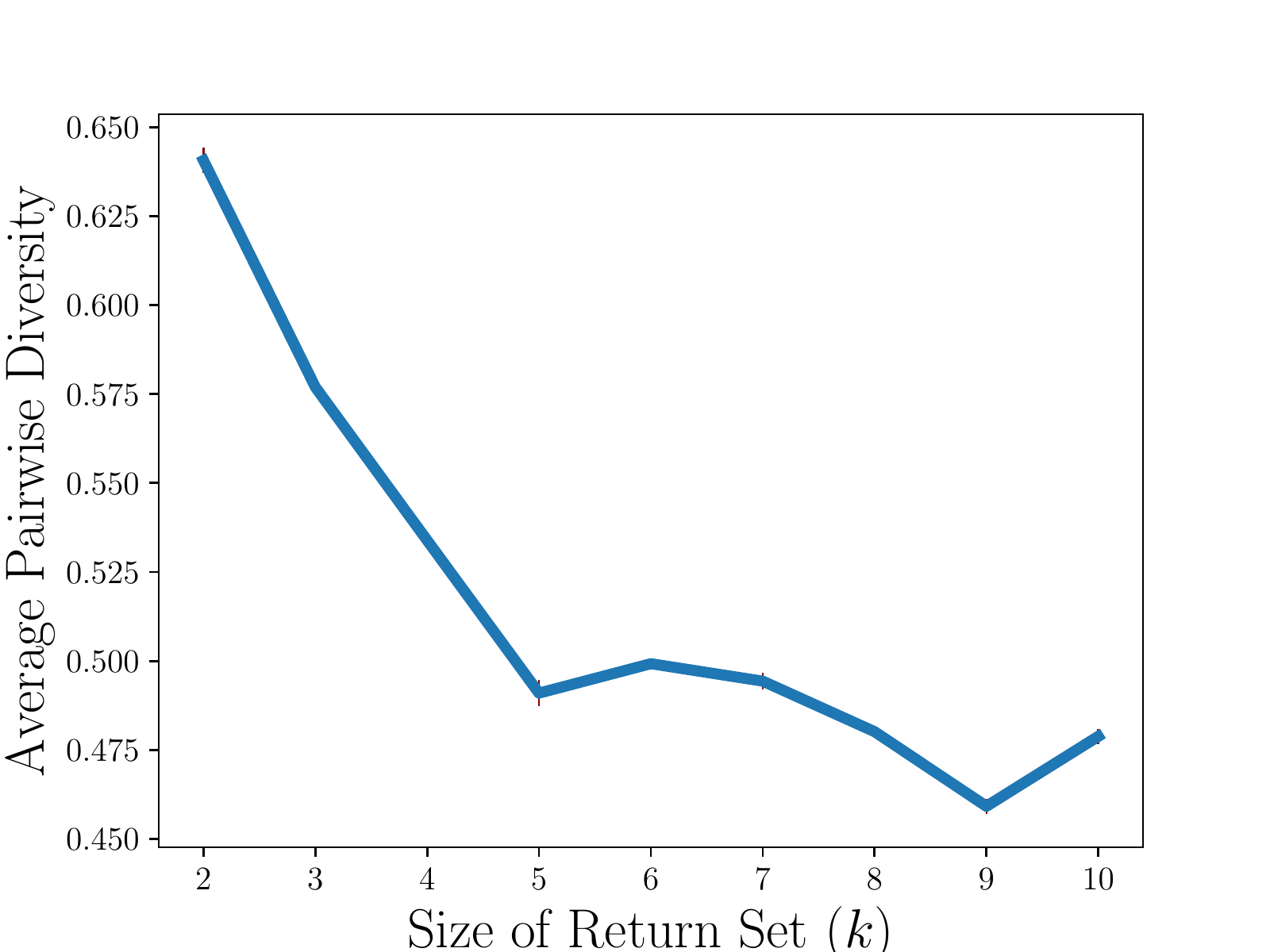} \\
      (b) \\
    \end{multicols}
    \caption{Average reward per policy and average pairwise diversity as a function of the size of the return set for ten  nine-room grid world trials. Here $\lambda = 8$ and $\alpha = .95$. (a) Mean and standard deviation of the average reward per policy. The reward from the optimal policy is displayed in gray. (b) Mean and standard deviation of the average pairwise diversity.}
    \label{NumPathResults}
\end{figure*}

\begin{figure*}[!htb]
    \centering
    \includegraphics[trim={2cm 4.5cm 1.5cm 4.5cm},clip, width = .63\textwidth]{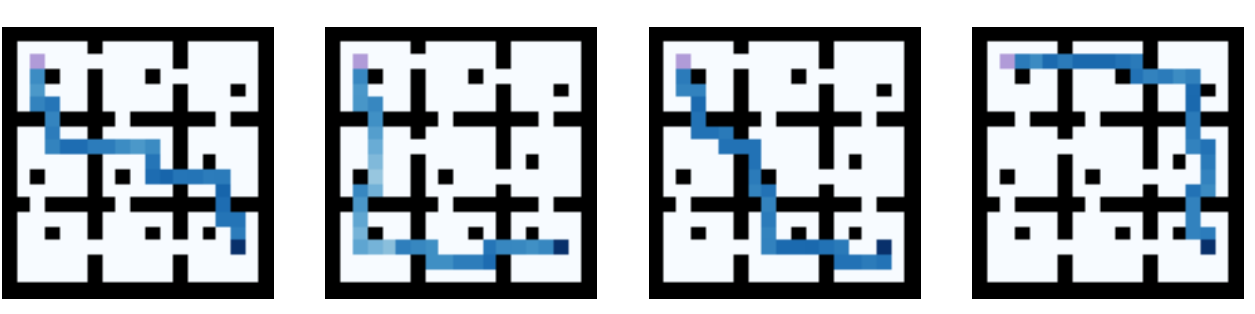} \\ 
    (a) \\
    \includegraphics[trim={2cm 5cm 1.5cm 5cm},clip, width = .95\textwidth]{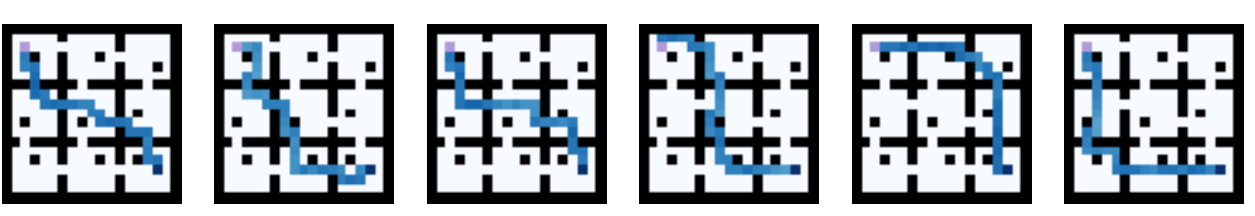} \\
    (b)
    \caption{Sample state occupancy maps from a nine-room grid-world trial with $\lambda = 8$, $\alpha = .95$, and two different return set sizes. Obstacles and walls are black, the initial state is red, the goal state is purple, and the state occupancy probability is blue. (a) Return set size is four. (b) Return set size is six. }
    \label{NumPathsOccupancy}
\end{figure*}

\subsection{Finding More Policies}
We show how varying the desired number of policies in the return set affects performance using the nine-room grid world. Here we set the value of the tradeoff parameter $\lambda = 8$ based on the results in the previous section and again set the correct transition parameter $\alpha = .95$. Figure \ref{NumPathResults} shows the average reward per policy and the average pairwise diversity of the policies found as a function of the size of the return set. As shown, there is a marked decrease in the average pairwise diversity as the size of the return set grows. This provides further evidence that the environment provides a natural limit on the number of meaningfully diverse policies that can be obtained. 

Figure \ref{NumPathsOccupancy} shows the state occupancy maps for the policies learned for a single trial and two different return set sizes. When the return set size is $k = 4$, our approach is able to find four policies that utilize distinct doors and rooms to reach the goal. When the return set size is $k = 6$, our approach finds five clearly distinct policies, including ones extremely similar to the four found with $k = 4$. However, policies 1 and 3 are very similar to each other. This behavior of increasingly diverse policies up to a threshold is consistent across the trials. We conclude that depending on the location of the obstacles and doors, our approach is able to find a number of distinct policies before reaching a limit that is related to the environment design.

\subsection{Correct Transition Parameter}

\begin{figure*}[!htb]
    \centering
    \begin{multicols}{2}
    \includegraphics[width = .45\textwidth]{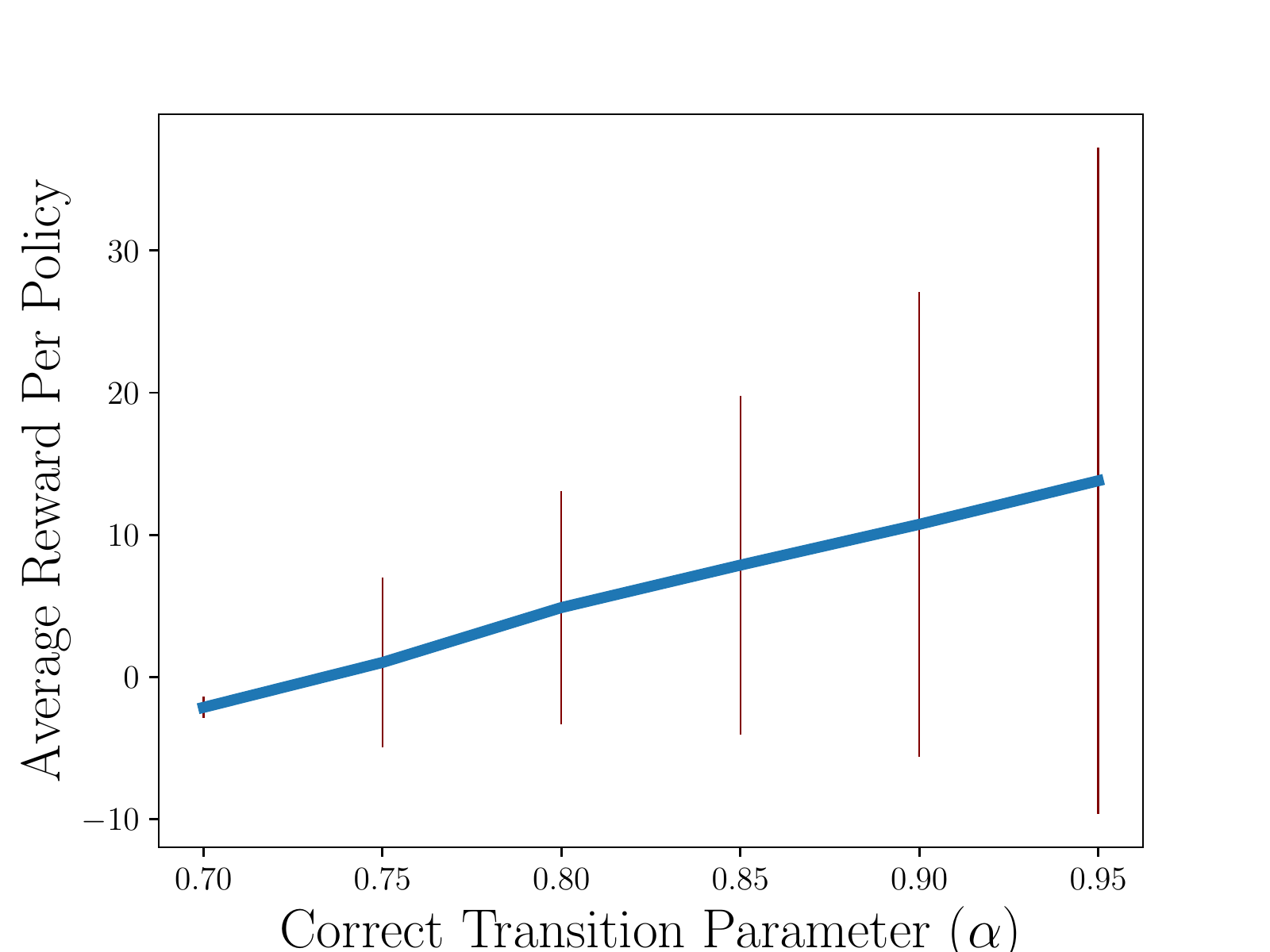} \\
      (a) \\
    \includegraphics[width = .45\textwidth]{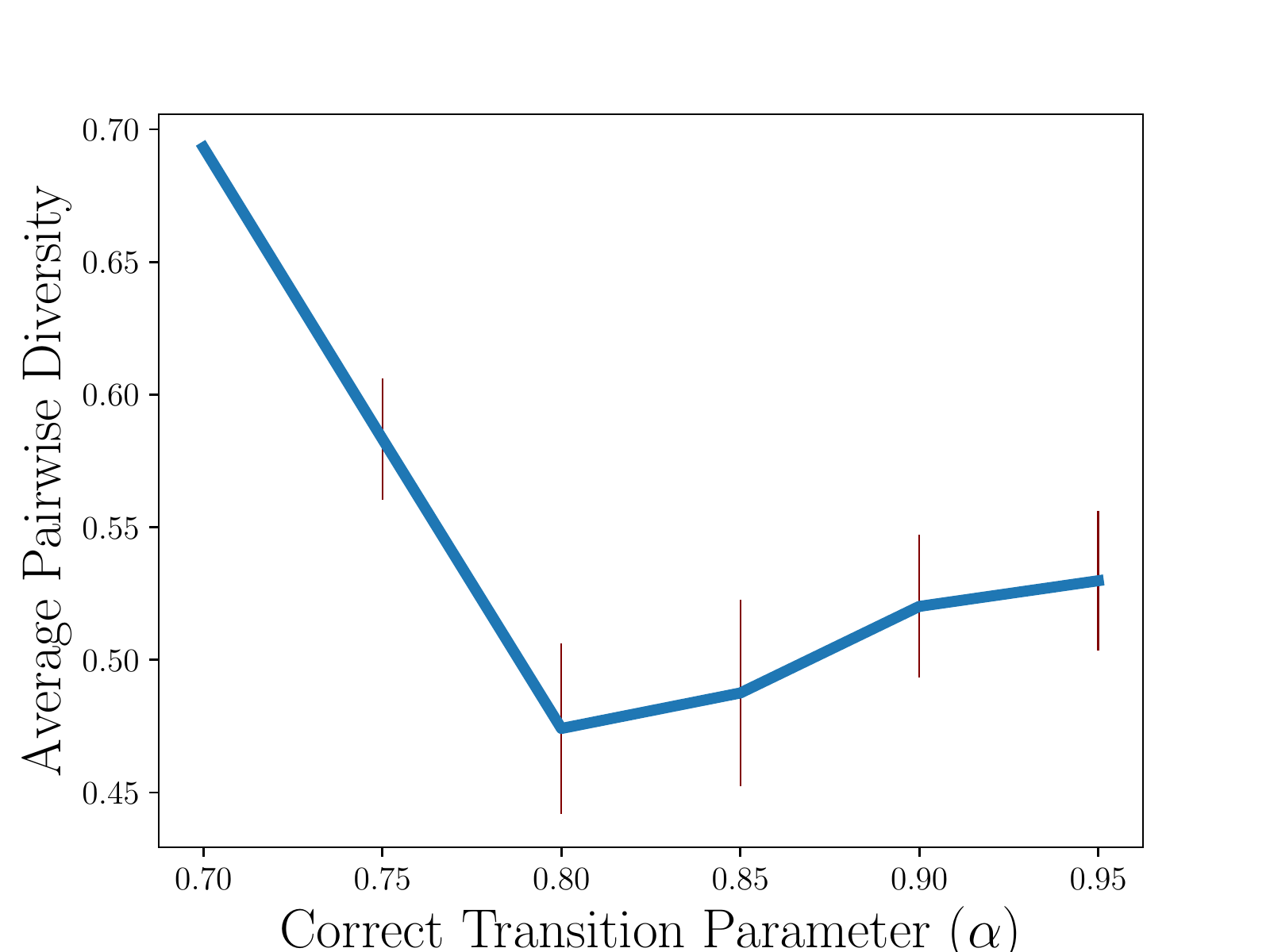} \\
      (b) \\
    \end{multicols}
    \caption{Average reward per policy and average pairwise diversity as a function of the correct transition parameter for twenty trials of the four-room grid world. Here $\lambda = 8$ and $k = 2$. (a) Mean and standard deviation of the average reward per policy..  (b) Mean and standard deviation of the average pairwise diversity.}
    \label{CorrectTrans}
\end{figure*}

Here we investigate the role of stochasticity in the performance of our approach. Figure \ref{CorrectTrans} shows the average reward per policy and average pairwise diversity as a function of the correct transition parameter, $\alpha$, in the four-room grid world. With low values of the correct transition parameter, the probability of hitting obstacles on the way to the goal is high, and the algorithm finds stationary policies with near-maximal diversity but a low reward. As the correct transition parameter increases it becomes optimal to find policies that reach the goal and the average diversity decreases before reaching a minima at $\alpha = .8$. The average diversity then begins to increase past this point as the decreased stochasticity leads to more distinct policies.

\section{Conclusion and Future Work}
In this work, we considered the problem of stochastic planning in situations where the objective function is known to be only partially specified. In this setting, we proposed generating a representative set of near-optimal policies with respect to the known objective. To that end, we formulated a nonlinear optimization problem that finds a small set of near-optimal and diverse policies. We showed that it is possible to efficiently solve the optimization problem using the Frank-Wolfe method and proved non-asymptotic convergence rates. We then compared the performance of the Frank-Wolfe method with projected gradient ascent and investigated the role of the hyperparameters using a series of navigation problems.

Our results show that the choice of the tradeoff parameter and the size of the return set play an important role in the performance of our approach. As the tradeoff parameter and the size of the return set increase, our approach is able to find increasing numbers of meaningfully distinct near-optimal policies up to a limit that is related to the structure of the environment. An interesting future extension of our approach would be investigating the utility of these near-optimal diverse strategies in generating effective collaboration between groups of autonomous agents. 

\bibliography{main.bib}

\end{document}